\documentclass{article} 
\usepackage{iclr2020_conference,times}
\usepackage{comment,soul}
\usepackage{subcaption}
\usepackage{color}
\usepackage[title]{appendix}

\usepackage[utf8]{inputenc} 
\usepackage[T1]{fontenc}    
\usepackage[scaled]{beramono}
\usepackage{hyperref}       
\usepackage{url}            
\usepackage{booktabs}       
\usepackage{amsfonts}       
\usepackage{nicefrac}       
\usepackage{microtype}      
\usepackage{multirow}
\usepackage{multicol}
\usepackage{verbatim}
\usepackage{subcaption}
\usepackage{xcolor}
\usepackage{enumitem}
\usepackage{graphicx}
\usepackage{wrapfig}
\usepackage[font={small}]{caption}

\newcommand{\code}[1]{{\footnotesize\texttt{#1} }}

\usepackage{amsthm, amssymb, amsmath}
\usepackage{mathtools}

\newcommand{\boldm}{\mathbf{m}}
\newcommand{\boldx}{\mathbf{x}}
\newcommand{\boldw}{\mathbf{w}}
\newcommand{\edge}[3]{(#1, #3, #2)}

\usepackage{listings}

\definecolor{background}{RGB}{249, 250, 250}
\definecolor{string}{RGB}{153, 199, 180}
\definecolor{comment}{RGB}{153, 153, 153}
\definecolor{normal}{RGB}{50, 50, 50}
\definecolor{identifier}{RGB}{102, 153, 204}
\definecolor{number}{RGB}{249, 174, 87}

\usepackage{makecell}

\usepackage{pifont}
\newcommand{\cmark}{\ding{51}}
\newcommand{\xmark}{\ding{55}}

\usepackage{floatrow}
\newfloatcommand{capbtabbox}{table}[][\FBwidth]
\usepackage{tabularx}

\title{Deep Graph Library: A Graph-Centric, Highly-Performant Package for Graph Neural Networks}

\iclrfinalcopy

\author{
Minjie Wang \textsuperscript{2,3},
Da Zheng \textsuperscript{1},
Zihao Ye \textsuperscript{2},
Quan Gan \textsuperscript{2},
Mufei Li \textsuperscript{2},
Xiang Song \textsuperscript{2},\\
\bf{
Jinjing Zhou \textsuperscript{2},
Chao Ma \textsuperscript{2},
Lingfan Yu \textsuperscript{3},
Yu Gai \textsuperscript{2},
Tianjun Xiao \textsuperscript{2},
Tong He \textsuperscript{2},}\\
\bf{
George Karypis \textsuperscript{1},
Jinyang Li \textsuperscript{3} \& Zheng Zhang \textsuperscript{2,4}} \\
\textsuperscript{1} Amazon Web Services, \textsuperscript{2} AWS Shanghai AI Lab \\
\textsuperscript{3} New York University, \textsuperscript{4} NYU Shanghai
}

\begin{document}

\maketitle

\begin{abstract}

Advancing research in the emerging field of deep graph learning requires new tools to support tensor computation over graphs. In this paper, we present the design principles and implementation of Deep Graph Library (DGL)\footnote{Project site: \url{https://www.dgl.ai}}. DGL distills the computational patterns of GNNs into a few generalized sparse tensor operations suitable for extensive parallelization. By advocating graph as the central programming abstraction, DGL can perform optimizations transparently. By cautiously adopting a framework-neutral design, DGL allows users to easily port and leverage the existing components across multiple deep learning frameworks. Our evaluation shows that DGL significantly outperforms other popular GNN-oriented frameworks in both speed and memory consumption over a variety of benchmarks and has little overhead for small scale workloads.


\end{abstract}
\section{Introduction}\label{s:intro}

Graph neural network (GNN) generalizes traditional deep learning to capture structural information in the data by modeling a set of node entities together with their relationships (edges). Its application range is broad, including molecules, social networks, knowledge graphs and recommender systems ~\citep{zitnik2018modeling,schlichtkrull2018modeling,hamilton2018embedding,pinsage}, or in general any datasets that have structural information. As a vibrant and young field, accelerating research on GNN calls for developing domain packages that are simultaneously flexible and powerful for researchers, and efficient and performant for real-world applications. 

Meeting both requirements are challenging: there are significant \emph{semantic} gaps between the tensor-centric perspective of today's popular deep-learning (DL) frameworks and that of a graph, and \emph{performance} gaps between the computation/memory-access patterns induced by the sparse nature of graphs and the underlying parallel hardware that are highly optimized for dense tensor operations.

This paper gives an overview of the design principles and implementation of Deep Graph Library (DGL), an open-source domain package specifically designed for researchers and application developers of GNN. Specifically, we make the following contributions:
\begin{itemize}[leftmargin=0.5cm]
    \item DGL distills the computational patterns of GNNs into a few user-configurable message-passing primitives; these primitives generalize sparse tensor operations and cover both the forward inference path and the backward gradient computing path. As such, they not only serve as the building blocks optimized for today's hardware, but also lay the foundation for future optimizations as well. In addition, DGL identifies and explores a wide range of parallelization strategies, leading to speed and memory efficiency. 
    \item DGL makes graph the central programming abstraction. The graph abstraction allows DGL to simplify user programming by taking full control of the messy details of manipulating graph data.
    \item A full GNN application takes more than a GNN model; the other components (e.g. data preprocessing and feature extraction) are outside the scope of DGL. As such, DGL strives to be as framework neutral as possible. DGL runs on top of PyTorch~\citep{pytorch}, TensorFlow~\citep{tensorflow}, MXNet~\citep{mxnet} and leverages their capability as much it can, while minimizing the effort it takes to port a model across frameworks. Many choices we made are applicable to other domain packages that share the same aspiration.
\end{itemize}

The rest of the paper is organized as follows. We first introduce the backgrounds about GNN message passing in Sec.~\ref{sec:gnn}. We formulate these computations as two computational patterns -- g-SpMM and g-SDMM, and discuss the parallelization strategies in Sec.~\ref{s:gspmm}. Sec.~\ref{sec:impl} describes the design and implementation of the DGL framework. We discuss some related works in Sec.~\ref{sec:related} and evaluate DGL in Sec.~\ref{sec:eval}.

\section{Graph neural networks and message passing}
\label{sec:gnn}



There have been a significant development in extending deep neural networks to non-euclidean data such as graphs and manifolds. Many efforts~\citep{scarselli2009,bruna2013spectral,defferrard2016convolutional,kipf2017semi,graphsage,gat} are made to formulate appropriate model architectures for learning on graphs, which gave birth to the Graph Neural Networks (GNNs) family. Recent studies~\citep{mpnn,graphnets} manage to unify different GNN variants into the \textbf{message passing paradigm}. Let $\mathcal{G(V, E)}$ be a graph with nodes $\mathcal{V}$ and edges $\mathcal{E}$; Let $\boldx_v\in\mathbb{R}^{d_1}$ be the feature for node $v$, and $\boldw_{e}\in\mathbb{R}^{d_2}$ be the feature for edge $\edge{u}{v}{e}$\footnote{We use a slightly different notation than a traditional $(u, v)$ pair, where $e$ is the ID associated with the edge.}. The message passing paradigm defines the following node-wise and edge-wise computation at step $t+1$:

\begin{align}
\label{eq:ef}
\text{Edge-wise: } \boldm_{e}^{(t+1)}& = \phi \left( \boldx_v^{(t)}, \boldx_u^{(t)}, \boldw_{e}^{(t)} \right), \edge{u}{v}{e}\in\mathcal{E}. \\
\label{eq:nf}
\text{Node-wise: } \boldx_v^{(t+1)} &= \psi \left(\boldx_v^{(t)}, \rho\left(\left\lbrace\boldm_{e}^{(t+1)} : \edge{u}{v}{e} \in \mathcal{E} \right\rbrace \right) \right).
\end{align}



In the above equations, $\phi$ is a \emph{message function} defined on each edge to generate a message by combining the edge feature with the features of its incident nodes;
$\psi$ is an \emph{update function} defined on each node to update the node feature by aggregating its incoming messages using the \emph{reduce function} $\rho$.
In GNNs, the message and update functions are parameterized by neural network modules, 
and $\rho$ can be any set function such as sum, mean, max/min, or even an LSTM network~\citep{graphsage}. 

\section{GNN message passing as generalized SpMM and SDDMM.}
\label{s:gspmm}

\newtheorem{thm}{Theorem}
\newtheorem{lem}{Lemma}
\newtheorem{coro}{Corollary}
\newtheorem{defin}{Definition}

\newcommand{\pp}[2]{\frac{\partial #1}{\partial #2}}
\newcommand{\dpp}[2]{\dfrac{\partial #1}{\partial #2}}
\newcommand{\gSDDMM}{\text{g-SDDMM}}
\newcommand{\gSpMM}{\text{g-SpMM}}

There is a strong connection between the message passing paradigm to sparse matrix operations. For example, given the node feature matrix $\mathbf{X}\in \mathbb{R}^{|\mathcal{V}|\times d}$ and the adjacency matrix $\mathbf{A}$ of a graph, the node-wise computation in the graph convolutional network (GCN)~\citep{kipf2017semi} is a sparse-dense matrix multiplication (SpMM) $\mathbf{Y}=\mathbf{A}\mathbf{X}$. For the edge-wise computation, many GNN models~\citep{gat, graphblas} calculate an attention weight on each edge. One popular formulation of calculating attention weight is by a dot product between the source and destination node features~\citep{vaswani2017attention}. This corresponds to a sampled dense-dense matrix multiplication (SDDMM) operation $\mathbf{W}=\mathbf{A}\odot (\mathbf{X}\mathbf{X}^T)$: semantically, it multiplies two dense matrices, followed by an element-wise multiplication with a sparse mask matrix, and output a sparse matrix.

An important characteristic of SDDMM is that it maps the representation of an edge's incident nodes to the representation on the edge. Similarly, SpMM aggregates the representation of a node's inbound edges into a node representation. Both of them can be extended.
Given a graph $\mathcal{G(V, E)}$,
\begin{itemize}[leftmargin=*]
    \item A \emph{generalized SDDMM} (g-SDDMM) defined on graph $\mathcal{G}$ with message function $\phi_m$ is a function
    $$
    \gSDDMM_{\mathcal{G}, \phi_m}: 
        \mathbb{R}^{|\mathcal{V}| \times d_1},
        \mathbb{R}^{|\mathcal{V}| \times d_2},
        \mathbb{R}^{|\mathcal{E}| \times d_3}
     \mapsto \mathbb{R}^{|\mathcal{E}| \times d_4}
    $$
    where the output edge representations $\mathbf{M} = \gSDDMM_{\mathcal{G}, \phi_m}\left(\mathbf{X}, \mathbf{Y}, \mathbf{W}\right)$ are computed from the edges' own features, as well as features of their incident nodes:
    $$
    \mathbf{m}_e =
    \phi_m\left( \mathbf{x}_u, \mathbf{y}_v, \mathbf{w}_e \right), \quad \forall \edge{u}{v}{e} \in \mathcal{E}.
    $$
    \item A \emph{generalized SpMM} (g-SpMM) defined on graph $\mathcal{G}$ with message function $\phi_z$ and reduce function $\rho$ is a function
    $$
    \gSpMM_{\mathcal{G}, \phi_z, \rho}: 
        \mathbb{R}^{|\mathcal{V}| \times d_1},
        \mathbb{R}^{|\mathcal{V}| \times d_2},
        \mathbb{R}^{|\mathcal{E}| \times d_3}
     \mapsto \mathbb{R}^{|\mathcal{V}| \times d_4}
    $$
    where the output node representations $\mathbf{Z} = \gSpMM_{\mathcal{G}, \phi_z, \rho}\left(\mathbf{X}, \mathbf{Y}, \mathbf{W}\right)$ are computed from the nodes' inbound edge features, the node features themselves, and the neighbor features:
    $$
    \mathbf{z}_v =
    \rho \left( \left\lbrace \phi_z\left( \mathbf{x}_u, \mathbf{y}_v, \mathbf{w}_e \right) : \edge{u}{v}{e} \in \mathcal{E} \right\rbrace \right), \quad \forall v \in \mathcal{V}.
    $$
\end{itemize}

These two primitives play an essential role in GNN computations; the forward path essentially applies a series of g-SpMM (and g-SDMM if attention is involved, as in GAT) to derive a stack of node representations. One can prove that the gradient of the objective function w.r.t. g-SDDMM and g-SpMM inputs can be expressed as another g-SDDMM and g-SpMM (see the supplementary materials for the full proof):

\begin{thm}
    \label{thm:mm}
    Given $\mathbf{M}=\text{g-SDDMM}_{\mathcal{G}, \phi_m}(\mathbf{X}, \mathbf{Y}, \mathbf{W})$ and $\mathbf{Z}=\text{g-SpMM}_{\mathcal{G}, \phi_z, \rho}(\mathbf{X}, \mathbf{Y}, \mathbf{W})$ and an objective function $\mathcal{L} = \ell(\mathbf{M}, \mathbf{Z})$.  Then
    \begin{itemize}
        \item The partial derivative
    $\pp{\mathcal{L}}{\mathbf{W}}$ can be computed by a g-SDDMM on graph $\mathcal{G}$.
        \item The partial derivative $\pp{\mathcal{L}}{\mathbf{X}}$ can be computed by a g-SpMM on the reverse graph $\Tilde{\mathcal{G}}(\mathcal{V}, \Tilde{\mathcal{E}}), \Tilde{\mathcal{E}} = \left\lbrace
        \edge{u}{v}{e} : \edge{v}{u}{e} \in \mathcal{E}\right\rbrace$.
        \item The partial derivative $\pp{\mathcal{L}}{\mathbf{Y}}$ can be computed by a g-SpMM on graph $\mathcal{G}$.
    \end{itemize}
\end{thm}

The benefits of such formulation are two. First, consolidating all the GNN computations into two patterns lays a foundation for system optimizations such as parallelization and auto-tuning. Second, g-SpMM naturally avoids generating intermediate storage for messages, and g-SDDMM avoids copying node representations to edges. Later, we will show that such fused computation is the key reason for a superior training speed and memory efficiency (Sec.~\ref{sec:eval}).

\subsection{Parallelizing g-SpMM and g-SDDMM on tensorized hardware}

Modern hardware like GPUs and TPUs utilizes large-scale multi-threading to achieve high throughput while hiding memory access latency. This requires the workload to have two characteristics. First, the computation-to-memory-access ratio must be high so that the cost of one memory operation is amortized over many floating point operations. Second, the workload should have sufficient parallelism to take advantage of the massive parallelization power in the hardware.. 

By these criteria, g-SpMM and g-SDDMM are inherently challenging workloads. First, each node's data is only used by its neighbors. With little data reuse, the resulting computation-to-memory-access ratio is low. Second, although there exist multiple ways to parallelize the g-SpMM and g-SDDMM operations -- by node, edge or feature, different strategies have pros and cons and there is no jack of all trades. Feature parallelization lets each thread compute on one feature and different ones can run in parallel. Although it is free of synchronization, the parallelism is limited by hidden size. For parallelization on nodes and edges, the optimal performance depends on numerous factors (see Table~\ref{tbl:parallel-summary}), including the preferred storage format, the specific computation pattern (i.e., g-SpMM or g-SDDMM), fine-grained synchronization to ensure atomicity (if required), degree of parallelism, and thus is ultimately both data and model dependent. DGL's current strategy is based on heuristics -- using node parallel for g-SpMM but edge parallel for g-SDDMM, and there is ample room to apply machine learning for performance optimization.

\begin{table}
\footnotesize
\caption{Summary of the node and edge parallel strategies.}
\label{tbl:parallel-summary}
\begin{tabularx}{\linewidth}{c|XX}
	\toprule
	& \textbf{Node Parallel} & \textbf{Edge Parallel} \\
	\midrule
	\textit{Schedule} & Each thread is in charge of the entire adjacency list of a node. & Each thread is in charge of one edge. \\
	\midrule
	\textit{Viability} & Any g-SpMM or g-SDDMM. & Any g-SDDMM and any g-SpMM with commutative and associative $\rho$. \\
	\midrule
	\textit{Preferred Format} & Compressed sparse row (CSR) due to fast lookup of adjacency list. & Coordinate list (COO) due to fast lookup of incident nodes. \\
	\midrule
	\textit{Need for synchronization} & No & No for g-SDDMM; g-SpMM requires atomic instructions for aggregating results to destination node memory. \\
	\midrule
	\textit{Workload Distribution} & Depend on node degrees & Balanced \\
	\midrule
	\textit{Parallelism} & Depend on number of nodes & Depend on number of edges \\
	\bottomrule
\end{tabularx}
\end{table}


\section{DGL system design}\label{sec:impl}
We now describe two important strategies we adopted in the development of DGL: 1) using graph as the central, user-friendly abstraction to allow deep optimization and 2) achieving maximum framework neutrality to enable seamless application integration.

\subsection{Graph as a first-class citizen}\label{s:design}

\begin{figure}
\centering
\includegraphics[width=\textwidth]{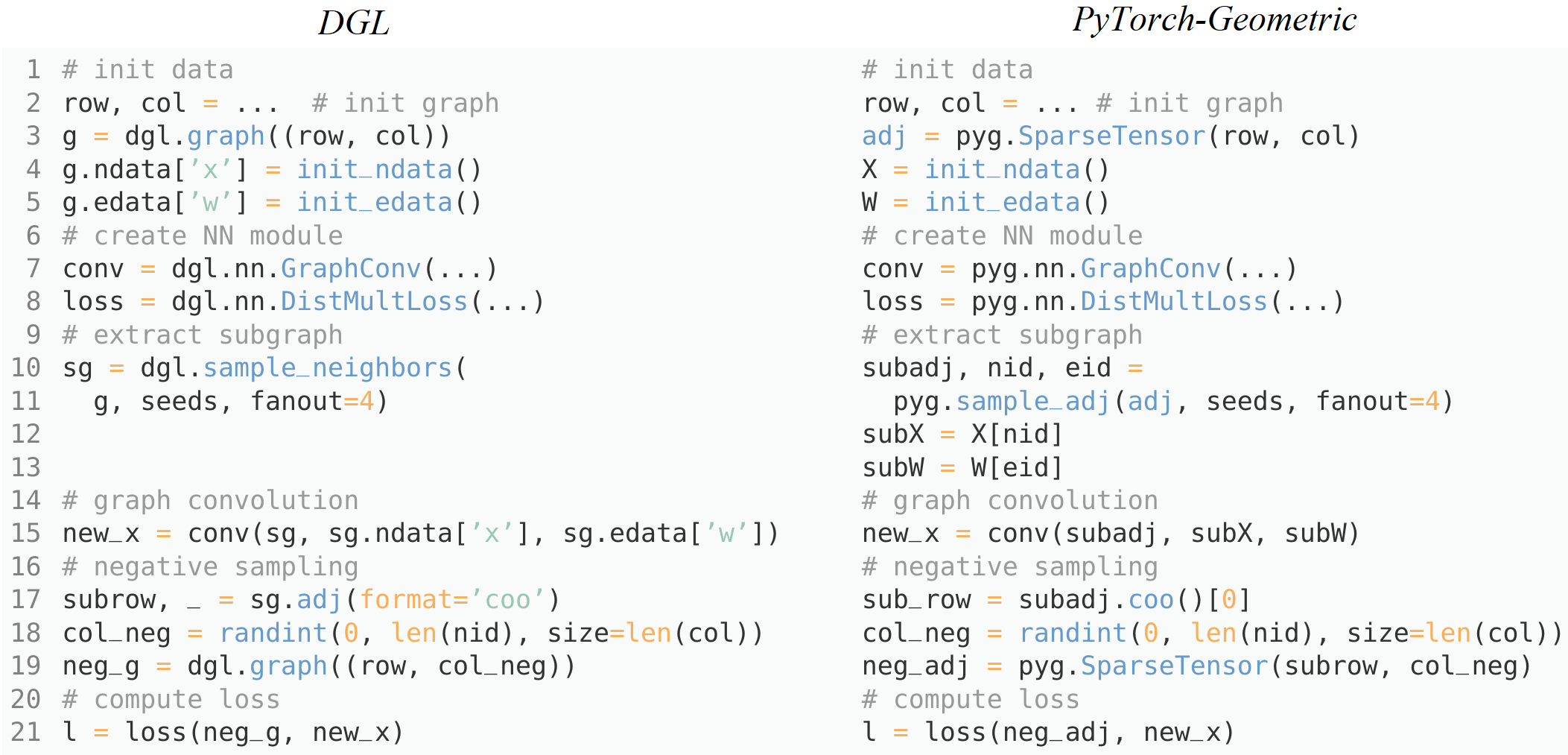}
\caption{Computing graph convolution on a subgraph in DGL and PyG.}
\label{fig:program}
\end{figure}

For a domain package designed for GNNs, it is natural to define graph as the central representation. While this is the consensus among different packages~\citep{pyg, graphnets, zhu2019aligraph, euler}, DGL differs in a number of places. First, it fully embraces an object-oriented programming style at the graph level. Second, recognizing the fact that research work on graphs are diverse and have developed a rich set of tools, it adopts an interface familiar to graph analytic experts. Third, the package exposes necessary low-level structures to advanced users when necessary. The fourth point, which we will discuss in Section~\ref{s:neutral}, is its extensive use of tensor data structure to allow seamless integration with base DL frameworks.

Figure~\ref{fig:program} compares the programming model of DGL and PyTorch Geometric (PyG)~\citep{pyg}. In DGL, \code{DGLGraph} is the key data structure created by the \code{dgl.graph} API (line 3).
Different models (e.g., \code{GraphConv} for graph convolution~\citep{kipf2017semi}, \code{GATConv} for graph attention model~\citep{gat}) operate on a \code{DGLGraph} directly (line 15); sampling, too, returns a \code{DGLGraph} object (line 10).
The returned subgraph object automatically extracts the features needed, saving the effort to manually slice from tensors (line 12--13). Consolidating graph operations in an object-oriented manner not only improves software consistency, but also enables performance enhancement transparent to users. For example, DGL automatically switches to use CSR or CSC formats for g-SpMM depending on whether it is forward or backward propagation, and uses COO for g-SDDMM.

Integrating graph with deep learning is a relatively young concept, but graph analytics has been a long-standing research field. There exist many sophisticated tools and packages. Many DGL APIs took inspiration from NetworkX~\citep{networkx} -- the NumPy-equivalent python package in graph analytics. Examples are topological query APIs implanted as class member functions such as \code{g.in\_degree} for getting node indegrees of a graph \code{g}. The two methods \code{g.edata} and \code{g.ndata} for accessing edge and node features are similarly inspired, with a dictionary-like interface that allows named tensors. Importantly, those APIs, while sharing naming convention with their NetworkX counterparts, have batched versions using tensor data structure. For example, NetworkX's \code{g.in\_degree} only supports querying the degree of one node at a time, while DGL supports querying multiple nodes by providing a tensor of node IDs. Finally, we note that these APIs are more than a matter of convenience. For instance, a \textit{heterogeneous graph} can have nodes or edges of different types, which can further have unaligned feature dimensions; it will be cumbersome to store them compactly in one tensor.

Finally, to maintain expressiveness and flexibility, it is important to expose internal structures so users can innovate beyond the APIs that \code{DGLGraph} offers. For instance, there have been a diverse number of studies~\citep{kotnis2017analysis, lukovnikov2017neural} on negative sampling strategies for the link prediction task. Users can craft these negative edges using the internal adjacency matrix of a DGLGraph object via the \code{g.adj} API (line 17--19 in Figure~\ref{fig:program}).

To define new GNN models, users can invoke the g-SpMM and g-SDDMM kernels via the \code{g.update\_all(}$\phi$, $\rho$\code{)} and \code{g.apply\_edges(}$\phi$\code{)} calls on a \code{DGLGraph}, with user-defined $\phi$ and $\rho$. In principle, a powerful compiler can parse any given functions and generate a fused kernel for execution. As the technique is yet to be developed (and thus is an active research), DGL provides a set of most common $\phi$ and $\rho$ as built-ins and generates kernels for each of the combination. Note all these kernels avoid materializing edge data to save memory, which is important for all GNN models where edge features are needed (Sec.~\ref{sec:eval}). 
For more complex user-defined functions (e.g., LSTM as a reduction function), DGL gathers node features to edges so users can compute messages in a batch. For the reduce phase, DGL groups nodes of the same degree into one bucket so that the received messages in each bucket can be stored in a dense tensor. DGL then invokes the user-defined reduce function repetitively on each bucket. This \textit{catch-all} strategy makes it easy for quickly prototyping model ideas on graphs of small sizes.



\subsection{Framework-neutral design}\label{s:neutral}
\begin{table}
\centering
\caption{The categorized number of lines of codes (LoC) need to change when porting a GNN layer in DGL from
PyTorch to TensorFlow. Code comments are excluded.}
\label{tbl:loc}
\footnotesize
\begin{tabularx}{0.8\linewidth}{p{5cm}|XXXp{1.5cm}}
  \toprule
 \multirow{2}{*}{\textbf{GNN Model}} & \multicolumn{3}{c}{\textbf{Change Type}} & \multirow{2}{*}{\shortstack{\bf Change / \\ \bf Total}} \\
  & \textbf{\textsc{I}} & \textbf{\textsc{II}} & \textbf{\textsc{III}} & \\ 
 \midrule
 GCN~\citep{kipf2017semi} & 4 & 12 & 14 & 30 / 103 \\ 
 GAT~\citep{gat} & 4 & 26 & 7 & 37 / 95 \\ 
 GraphSage~\citep{graphsage} & 4 & 13 & 8 & 25 / 85 \\
 GIN~\citep{xu2018powerful} & 4 & 4 & 0 & 8 / 37 \\
 SGC~\citep{sgc} & 4 & 4 & 4 & 12 / 50 \\
 \bottomrule
\end{tabularx}
\end{table}

A natural way to build a domain package is to build it on top of \textit{one} of the DL frameworks (e.g., PyTorch, TensorFlow, and MXNet). 
These mature frameworks already provide high-performance differentiable dense tensor operators, rich neural network modules and optimizers; there is scarcely any reason to reinvent the wheel. DGL makes a conscious decision to extend \textit{multiple} frameworks and, consequently, to be as framework-neutral as possible. Our belief is that a real-world, end-to-end GNN application will require other modules outside of GNN, and that they can be, or have already been, implemented in any framework of users’ choice. In addition, users may favor a particular framework simply because of its unique features.

Note that being framework-neutral is different from framework-agnostic. That is, while DGL has both PyTorch and TensorFlow backends, a PyTorch DGL model still needs to be modified if it is to be run in TensorFlow. Being completely framework-agnostic requires putting a shim over all conceivable operators across different frameworks, the cost of which is prohibitive. Instead, we adopt a practical approach and reduce framework dependencies as much as possible, while providing clear guidelines as where the changes are to be made. Importantly, DGL can achieve a high degree of framework neutrality in part due to the abstraction and implementation of \code{DGLGraph}. As we shall describe below and quantify through the models that we have implemented, the changes are often local and trivial. 


A complete GNN application includes data loading and preprocessing, GNN model setup, training loop and evaluation. In theory, they are all framework dependent. The goal of our design is to make model specification as portable as possible. Versions of the same model for different frameworks differ in three categories: \textsc{\bf (I)} model class inheritance (e.g., using \code{tensorflow.keras.layer.Layer} instead of \code{torch.nn.Module}); \textsc{\bf (II)} sub-modules used inside the model and parameter initialization; \textsc{\bf (III)} framework-specific operators (e.g., using \code{tensorflow.matmul} instead of \code{torch.matmul}). A mini-porting guide and example codes are included in the supplementary materials.

Table~\ref{tbl:loc} shows the number of lines of code to change when we port several GNN layers from PyTorch to TensorFlow in DGL. They account for roughly 20\%--40\% of the entire model implementations. Most of them are trivial modifications and are easy for developers versed in both frameworks. Importantly, all graph-related operations are unified and stay identical in different versions.

To achieve this level of framework neutrality with a minimum performance impact, DGL must decide what functionalities to delegate and re-direct to base frameworks, and otherwise judiciously take over the control. We summarize the main principles below; most of them shall be applicable to other domain packages that share the same aspiration.


\paragraph{Owning the minimum \& the critical.} From our experience, a domain package must maintain control at places where performance or usability matters the most. For DGL, it means sparse tensor storage management and operations. This leads to the decision of defining \code{DGLGraph} (see Section~\ref{s:design}). The first release of DGL used dense operations from frameworks to express sparse operations in GNNs and performed poorly. We decided to implemented sparse operations ourselves.

\paragraph{Leverage and delegate otherwise.} Most of DGL's APIs take framework tensors as input and perform dense operations on them. DGL defines a shim to map dense tensor operations to their framework-specific implementations. For instance, summing up the hidden state of all nodes is a common readout function for graph-level classification. To batch this readout operation we define a shim function \code{unsorted\_1d\_segment\_sum}, which translates to \code{unsorted\_segment\_sum} in Tensorflow and \code{scatter\_add} in PyTorch. Such remapping is in spirit similar to ONNX \citep{onnx}, but is designed specifically for DGL. 



To enable auto-differentiation, all computation involving node/edge features must be expressed with differentiable functions and operators. DGL defines custom functions that directly take the \texttt{DGLGraph} as well as node/edge features as inputs, and return node/edge features as outputs.  These operators are then registered as PyTorch/Tensorflow/MXNet auto-differentiable functions.


DGL also takes advantage of DLPack~\citep{dlpack}, an open-source in-memory tensor structure specification for sharing tensors among deep learning frameworks, to directly process and return DL framework tensors without data copying. Many frameworks, including Pytorch, MXNet, and TensorFlow, natively support DLPack.

The above functionality calls for memory allocation and management. DGL delegates memory management to the base frameworks. A base framework usually implements sophisticated memory management (e.g., to reduce memory allocation overhead and memory consumption), which is especially important for GPU memory. Because the output shape of a graph kernel is well determined before execution, DGL calculates the output memory size, allocates memory from the frameworks for the output tensors and pass the memory to the kernel for execution.

\section{Related Work}
\label{sec:related}

Due to the rising interests in GNNs, frameworks designed specifically for expressing and accelerating GNNs are developed. PyTorch-Geometric (PyG)~\citep{pyg} is an extension for geometric deep learning to the PyTorch framework. PyG's programming model is centered around sparse tensor abstraction. During message passing, it first \textit{gathers} node features to edges, applies user-defined message function and then \textit{scatters} them to the target nodes for aggregation. This scatter-gather pattern is inefficient due to generating large intermediate message tensors. GraphNet~\citep{graphnets} and AliGraph~\citep{zhu2019aligraph} are two TensorFlow packages for building GNN models. Both frameworks allow customizable message functions but the reducers are limited to TensorFlow's operators for segment reduction. Euler~\citep{euler} focuses on sampling-based mini-batching training on large graphs but lacks GPU support. NeuGraph~\citep{ma2019neugraph} accelerates GNN training by partitioning graphs to multiple GPUs. All of these systems are tied to one specific base DL framework.

There is a long line of work for optimizing sparse matrix operators such as sparse matrix-vector multiplication (SpMV) or sparse matrix-matrix multiplication (SpMM) on both CPU and GPU. These techniques range from studying and innovating new sparse matrix formats~\citep{bell2008efficient,filippone2017sparse}, advanced parallel pattern~\citep{yang2018design} to tiling and reordering~\citep{yang2011fast,baskaran2009optimizing} in the context of graph analytics~\citep{ashari2014fast,wang2016gunrock} or scientific computing applications~\citep{leveque2007finite}, just to name a few. Our work formally connects the area to GNN applications through the notions of generalized SpMM and SDDMM. We present the emerging challenges and hope to open up new innovations in this domain.

\section{Evaluation}
\label{sec:eval}

In this section, we compare DGL with other popular GNN frameworks: PyTorch-Geometric v1.5.0 (PyG) with PyTorch v1.5.0 as backend and GraphNet v1.1.0 with TensorFlow v2.2.0 as backend.\footnote{Benchmark scripts are available at \url{https://github.com/dglai/dgl-0.5-benchmark/}}


\subsection{Training speed}\label{ss:eval-speed}

We consider two benchmark tasks: node classification and link prediction, and two training methods: full graph training and mini-batch training. Node classification datasets include the \textsc{Reddit} graph from~\citep{graphsage}, the \textsc{ogbn-arxiv}, \textsc{ogbn-protein}, and \textsc{ogbn-product} graphs from the Open Graph Benchmarks (OGB)~\citep{ogb}. For link prediction, we use benchmarks from \textsc{MovieLens(ML)}~\citep{movielens}, the \textsc{ogbl-citation} and \textsc{ogbl-ppa} graphs from OGB.

To demonstrate the generality of DGL's optimizations, we benchmark a variety of state-of-the-art GNN models, including GCN~\citep{kipf2017semi}, GraphSAGE~\citep{graphsage}, GAT~\citep{gat}, R-GCN~\citep{rgcn} and GCMC~\citep{berg2017graph}. All the node classification tasks use cross entropy loss on the node representations learned by the GNN models while the link prediction tasks perform edge predictions by computing the dot-product of the source and destination node representations. For mini-batch training, we experiment with two sampling methods: neighbor sampling (NS)~\citep{graphsage} and cluster sampling (CS)~\citep{clustergcn}. The supplementary material includes additional details about the datasets and model configurations.

All experiments record the training time of one epoch averaged over 10 runs. For full graph training, we measure the training time on both CPU and GPU. The testbeds are one AWS EC2 p3.2xlarge instance (one NVidia V100 GPU with 16GB GPU RAM and 8 VCPUs) and one m5n.16xlarge instance (64 VCPUs and 256GB RAM) for experiments on GPU and CPU respectively. For mini-batch training, we perform sampling on CPU and copy the sampled subgraphs and features to GPU for training. The testbed is a p3.2xlarge instance.

\begin{figure}
\begin{floatrow}
\capbtabbox[0.45\textwidth]{
	\centering
	\scriptsize
	\begin{tabular}{c|c|c|c|c|c}
		\toprule
		\multirow{2}{*}{\textbf{Dataset}} & \multirow{2}{*}{\textbf{Model}} & \multicolumn{2}{c}{CPU} & \multicolumn{2}{c}{GPU} \\
		 &  & \textbf{DGL} & \textbf{PyG} & \textbf{DGL} & \textbf{PyG} \\
		\midrule
		\multicolumn{6}{c}{Node Classification} \\
		\midrule
		\textsc{Reddit} & SAGE & \textbf{13.80} & 99.47 & 0.432 & \textbf{0.403} \\
		\textsc{Reddit} & GAT & \textbf{9.15} & OOM & 0.718 & OOM \\
		\textsc{ogbn-arxiv} & SAGE & \textbf{3.31} & 8.389 & 0.104 & \textbf{0.098} \\
		\textsc{ogbn-arxiv} & GAT & \textbf{1.237} & 43.21 & \textbf{0.086} & 0.234 \\
		\textsc{ogbn-protein} & R-GCN & \textbf{26.31} & 373.8 & \textbf{0.706} & OOM \\
		\midrule
		\multicolumn{6}{c}{Link Prediction} \\
		\midrule
		\textsc{ML-100k} & GCMC & \textbf{0.064} & 1.569 & 0.021 & \textbf{0.012} \\
		\textsc{ML-1m} & GCMC & \textbf{0.351} & 40.47 & \textbf{0.045} & 0.103 \\
		\textsc{ML-10m} & GCMC & \textbf{5.08} & OOM & \textbf{0.412} & OOM \\
		\bottomrule
	\end{tabular}
}{
\caption{Epoch running time in seconds (full graph training). OOM means out-of-memory.}
\label{tbl:full-eval}
}
\hspace{0.9cm}
\capbtabbox[0.45\textwidth]{
	\centering
	\scriptsize
	\begin{tabular}{c|c|c|c}
		\toprule
		\textbf{Dataset} & \textbf{Model} & \textbf{DGL} & \textbf{PyG} \\
		\midrule
		\multicolumn{4}{c}{Node Classification} \\
		\midrule
		\textsc{Reddit} & SAGE w/ NS & \textbf{19.90} & 20.45 \\
		\textsc{Reddit} & GAT w/ NS & \textbf{21.07} & 21.89 \\
		\textsc{ogbn-product} & SAGE w/ NS & \textbf{33.34} & 35.00 \\
		\textsc{ogbn-product} & GAT w/ NS & \textbf{67.0} & 187.0 \\
		\textsc{ogbn-product} & SAGE w/ CS & 8.887 & \textbf{8.614} \\
		\textsc{ogbn-product} & GAT w/ CS & \textbf{14.50} & 58.36 \\
		\midrule
		\multicolumn{4}{c}{Link Prediction} \\
		\midrule
		\textsc{ogbl-citation} & GCN w/ CS & \textbf{5.772} & 6.287 \\
		\textsc{ogbl-citation} & GAT w/ CS & \textbf{6.081} & 8.290 \\
		\textsc{ogbl-ppa} & GCN w/ CS & \textbf{5.782} & 6.421 \\
		\textsc{ogbl-ppa} & GAT w/ CS & \textbf{6.224} & 8.198 \\
		\bottomrule
	\end{tabular}
}{
\caption{Epoch running time in seconds for mini-batch training using neighbor sampling (NS) and cluster sampling (CS).}
\label{tbl:sample-eval}
    }
\end{floatrow}
\end{figure}

Table~\ref{tbl:full-eval} shows the results of full graph training. For GraphSAGE on GPU, both DGL and PyG use the vendor-provided cuSPARSE~\citep{cusparse} library for computing SpMM so the performance is similar. DGL is slower by a small margin (2--11\%) due to framework overhead. For GAT, DGL is 1.68$\times$ faster than PyG on \textsc{ogbn-arxiv} because DGL's g-SpMM kernel avoids generating message tensors while PyG's scatter-gather kernel does. This also explains PyG running out-of-memory on \textsc{ogbn-protein} due to the graph being the densest one among all and having edge features. Link prediction benchmarks show similar results. DGL is slower on small graphs (e.g., \textsc{ML-100K}) but is 1.83$\times$ faster on \textsc{ML-1M}. DGL can train on \textsc{ML-10M} while PyG runs out of memory. On CPU, DGL outperforms PyG on all benchmarks by 1.9$\times$--64$\times$. This is attributing to the high CPU utilization (50\%) of DGL's g-SpMM and g-SDDMM kernels using multi-threading compared with PyG's (only 10\%). The large gap of \textsc{ML-1M} is further caused by the huge intermediate message tensor, resulting in a lot of time spent in memory traffic.

Table~\ref{tbl:sample-eval} evaluates the performance of mini-batch training. Compared with full graph training, the total training time also depends on the cost in sample preparation including the sampling operations and data movement from CPU to GPU. For neighbor sampling (NS), sample generation and data movement can occupy up to 85\% of the total training time, which explains why DGL and PyG have similar performance. By contrast, cluster sampling (CS) is much faster and the benefit from DGL's optimized kernels gives an 1.56$\times$ speedup for training GAT. DGL beats PyG in all link prediction benchmarks due to the use of g-SDDMM kernel in computing predictions on edges.

\subsection{Memory consumption}\label{ss:eval-mem}
\begin{figure}[t]
\begin{floatrow}
\ffigbox[0.66\textwidth]{
    \centering
    \begin{subfigure}[t]{0.47\linewidth}
    	\centering
    	\includegraphics[width=\textwidth]{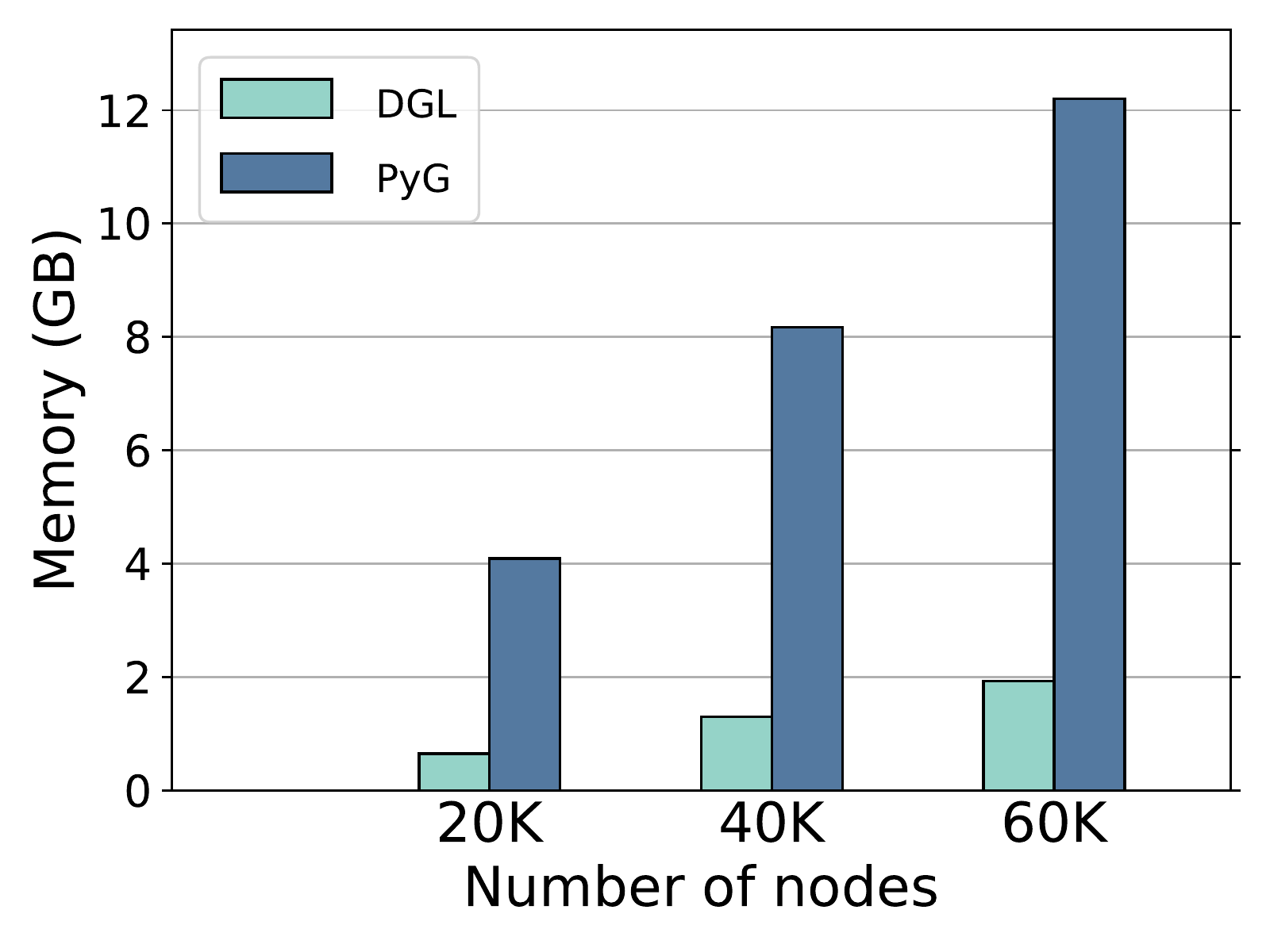}
    	\caption{}
    	\label{fig:mem-scale}
    \end{subfigure}
    \hfill
    \begin{subfigure}[t]{0.47\linewidth}
        \centering
        \includegraphics[width=\textwidth]{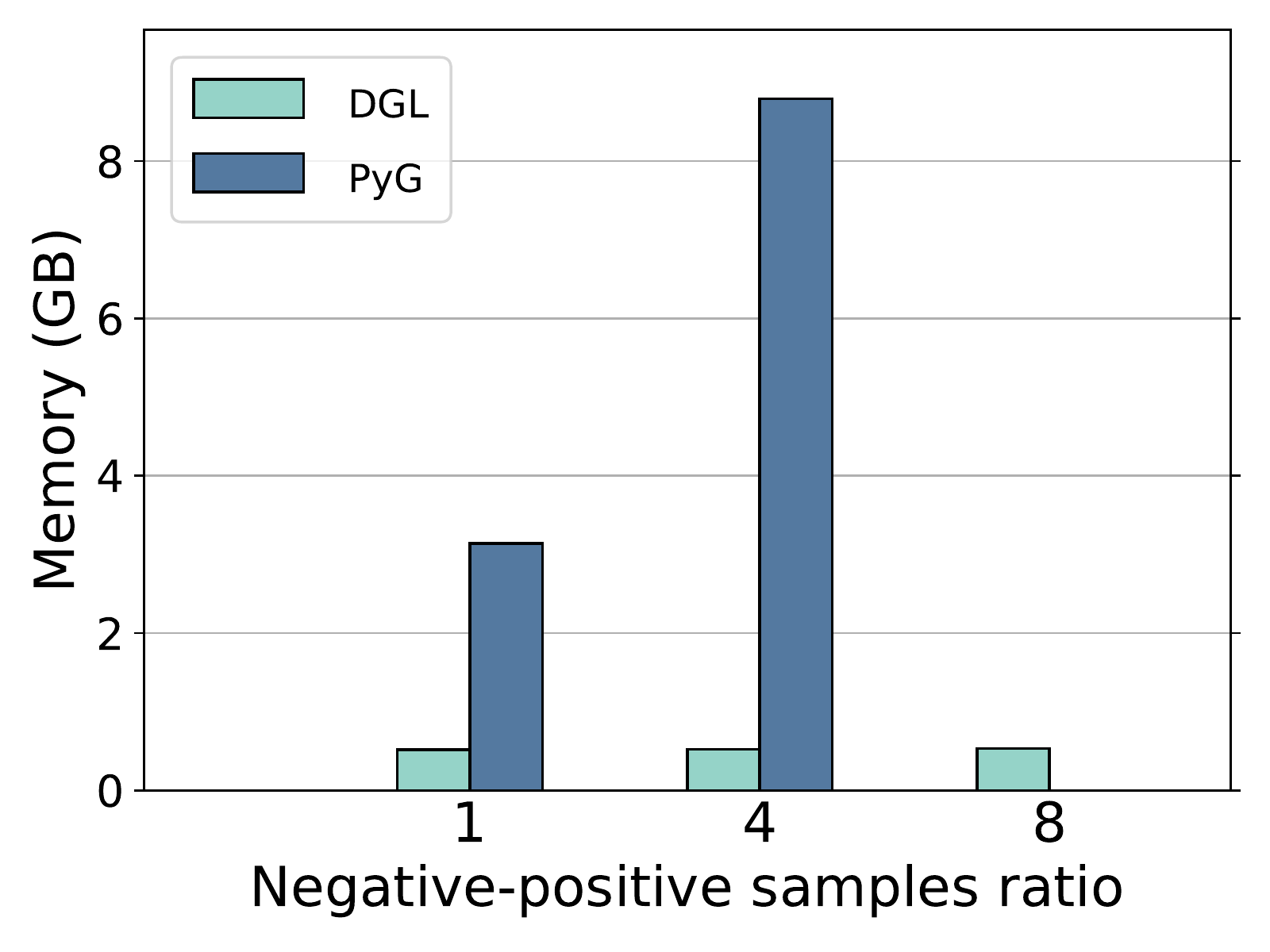}
        \caption{}
    	\label{fig:mem-neg}
    \end{subfigure}
}{
    \caption{Memory usage of PyG and DGL. (a) GAT on synthetic graphs; (b) GCN w/ CS on  \textsc{ogbl-citation}.}
    \label{fig:mem}
}
\hfill
\floatbox{figure}[.33\textwidth][\FBheight][t]{
    \centering
    \includegraphics[width=\linewidth]{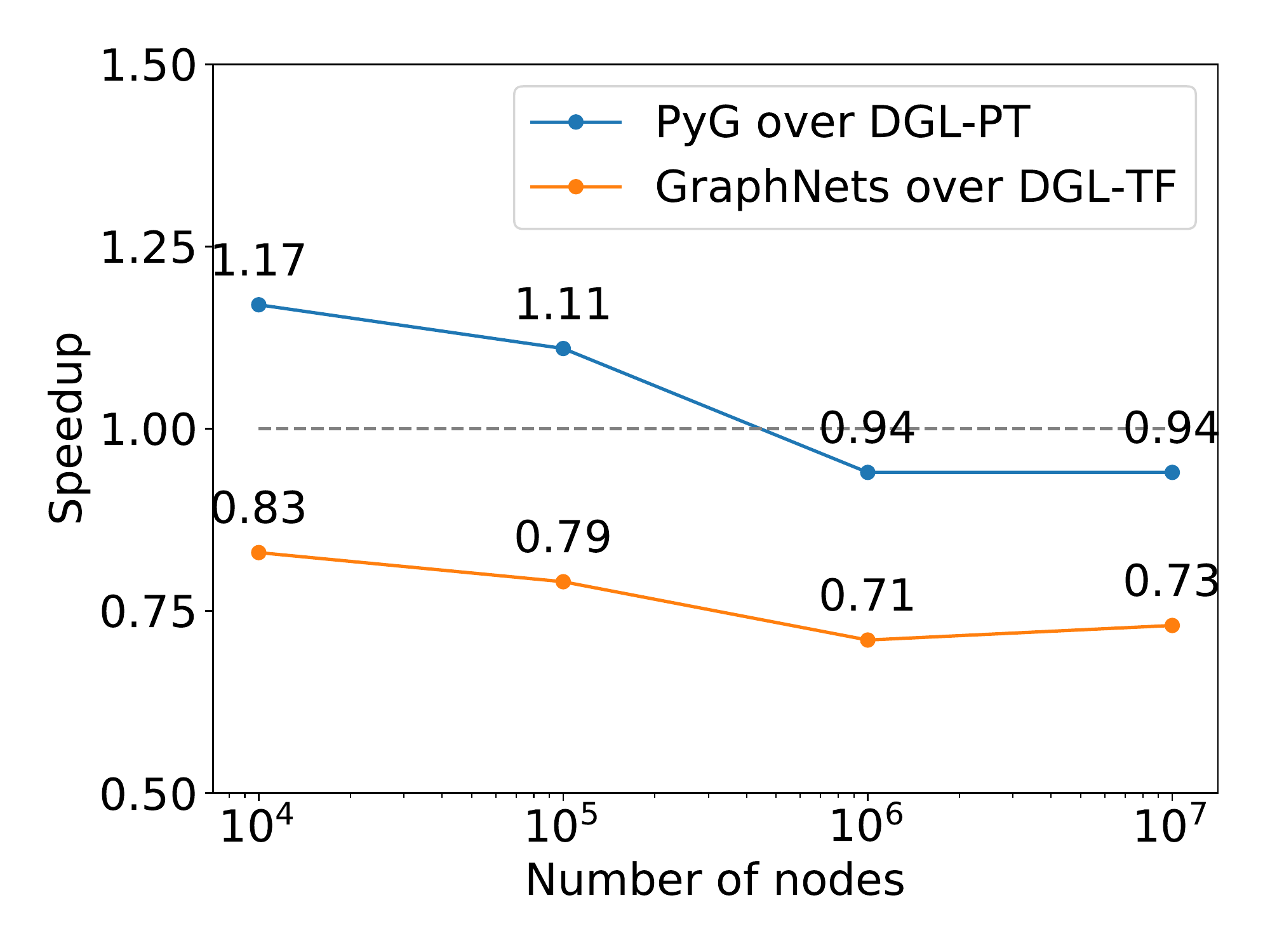}
}{
    \caption{Raw framework speedup of PyG and GraphNets over DGL.}
    \label{fig:framework_overhead}
}
\end{floatrow}
\end{figure}

To illustrate the advantage of DGL's g-SpMM and g-SDDMM kernels in reducing memory traffic, we further studied the memory usage of DGL and PyG. We trained a 3-layer GAT model with one attention head on a set of synthetic graphs of different scales. The average degree is fixed at 20 so the number of edges grows linearly with the number of nodes. Figure~\ref{fig:mem-scale} shows that PyG consumes 6.3$\times$ more memory than DGL and runs out of memory on graphs of more than 60K nodes. DGL manages to keep a low memory footprint due to its g-SpMM kernel fusing the message computation with aggregation. We further investigate the case of link prediction using a 3-layer GCN model with cluster sampling on the \textsc{ogb-citation} dataset. The model computes a prediction on each edge by performing a dot-product of its source and destination node representations, which is a typical SDDMM operation. Figure~\ref{fig:mem-neg} shows the memory usage by increasing the number of negative edges per positive ones at each mini-batch. DGL's memory consumption stays the same regardless of the number of negative samples while PyG quickly runs out of memory. Since negative sampling is universal in link prediction tasks, we expect the phenomenon to appear in other benchmarks as well.




\subsection{Framework overhead}\label{ss:eval-overhead}

We compared DGL's framework overhead with both PyG and GraphNets using PyTorch and TensorFlow as backends, respectively. In order to eliminate the impact of message passing kernels, we trained an one-layer GCN over a synthetically generated chain graph and measured the epoch time. We varied the number of nodes and plotted the speedup of PyG and GraphNets over DGL in Figure~\ref{fig:framework_overhead}. Ideally, the speedup should be one. We observed a 17\% overhead compared with PyG when the graph is very small, but as the graph size increases, the overhead becomes negligible. The overhead is due to DGL registering message passing kernels via Python to keep implementation independent of frameworks while PyG can register them in C++. Interestingly, DGL-TF is faster than GraphNets, which uses native TensorFlow operators, demonstrating the viability of a framework-neutral package with low overhead.

\section{Conclusion}

We present Deep Graph Library (DGL), a system specialized for deep learning models on graphs. DGL identifies the connection between sparse matrix computation and the message passing paradigm in graph neural networks, and consolidates these operations into generalized sparse-dense matrix multiplication (g-SpMM) and sampled dense-dense matrix multiplication (g-SDDMM). DGL explores a wide range of parallelization strategies, leading to its superior speed and memory efficiency. DGL presents two design principles. By having graph as the core programming abstraction, DGL can hide cumbersome details from users and perform optimization transparently. DGL's lessons in designing a framework-neutral domain package with low overhead shall also be applicable to other packages of the same kind.

\bibliography{ref}
\bibliographystyle{iclr2020_conference}

\newpage

\begin{appendices}
\section{Nomenclatures}

In the Appendix we adopt the following nomenclatures for representing different mathematical objects:

\begin{itemize}
    \item $a$: a scalar
    \item $\mathbf{a}$: a (column) vector
    \item $\mathbf{A}$: a matrix
    \item $\mathbf{a}_i$: the $i$-th row of matrix $\mathbf{A}$
    \item $\mathbb{R}^{m\times n}$: the set of real matrices with $m$ rows and $n$ columns.
    \item $\lbrace x : y \rbrace$ the set with all mathematical objects $x$ that satisfies condition $y$.
    \item $f(\cdot, \cdot, \dots)$: a function.
    \item $f : \mathcal{X} \mapsto \mathcal{Y}$: a function that maps from set $\mathcal{X}$ to $\mathcal{Y}$
    \item $\pp{\mathbf{y}}{\mathbf{x}}$ or $\nabla_\mathbf{x} \mathbf{y}$: the Jacobian of $\mathbf{y}$ with respect to $\mathbf{x}$, in numerator layout.
    \item $\pp{f}{\mathbf{x}}$: the partial derivative of scalar function $f$ with respect to vector $\mathbf{x}$, in numerator layout.  Note that in numerator layout $\pp{f}{\mathbf{x}} =
    \begin{bmatrix}
    \pp{f}{x_1} & \pp{f}{x_2} & \cdots
    \end{bmatrix}
    $ is a row vector.
    \item $\pp{f}{\mathbf{A}}$: the partial derivative of scalar function $f$ with respect to matrix $\mathbf{A}$.
    \item $\left[\mathbf{A} ; \mathbf{B} ; \cdots\right] = \begin{bmatrix}
    \mathbf{A} & \mathbf{B} & \dots
    \end{bmatrix}$: horizontal concatenation of matrices.
\end{itemize}

\section{Gradient of g-SpMM and g-SDDMM}

We go by reviewing the definition of g-SpMM and g-SDDMM:

\begin{defin}
    A generalized SDDMM (g-SDDMM) defined on graph $\mathcal{G}$ with message function $\phi_m$ is a function
    $$
    \gSDDMM_{\mathcal{G}, \phi_m}: 
        \mathbb{R}^{|\mathcal{V}| \times d_1},
        \mathbb{R}^{|\mathcal{V}| \times d_2},
        \mathbb{R}^{|\mathcal{E}| \times d_3}
     \mapsto \mathbb{R}^{|\mathcal{E}| \times d_4}
    $$
    where the output edge representations $\mathbf{M} = \gSDDMM_{\mathcal{G}, \phi_m}\left(\mathbf{X}, \mathbf{Y}, \mathbf{W}\right)$ are computed from the edges' own features, as well as features of their incident nodes:
    $$
    \mathbf{m}_e =
    \phi_m\left( \mathbf{x}_u, \mathbf{y}_v, \mathbf{w}_e \right), \quad \forall \edge{u}{v}{e} \in \mathcal{E}.
    $$
\end{defin}

\begin{defin}
    A generalized SpMM (g-SpMM) defined on graph $\mathcal{G}$ with message function $\phi_z$ and reduce function $\rho$ is a function
    $$
    \gSpMM_{\mathcal{G}, \phi_z, \rho}: 
        \mathbb{R}^{|\mathcal{V}| \times d_1},
        \mathbb{R}^{|\mathcal{V}| \times d_2},
        \mathbb{R}^{|\mathcal{E}| \times d_3}
     \mapsto \mathbb{R}^{|\mathcal{V}| \times d_4}
    $$
    where the output node representations $\mathbf{Z} = \gSpMM_{\mathcal{G}, \phi_z, \rho}\left(\mathbf{X}, \mathbf{Y}, \mathbf{W}\right)$ are computed from the nodes' inbound edge features, the node features themselves, and the neighbor features:
    $$
    \mathbf{z}_v =
    \rho \left( \left\lbrace \phi_z\left( \mathbf{x}_u, \mathbf{y}_v, \mathbf{w}_e \right) : \edge{u}{v}{e} \in \mathcal{E} \right\rbrace \right), \quad \forall v \in \mathcal{V}.
    $$
\end{defin}

We also review the formal definition of the reverse graph.

\begin{defin}
    Given the graph $\mathcal{G} = (\mathcal{V}, \mathcal{E})$, the reverse graph $\Tilde{\mathcal{G}} = (\mathcal{V}, \Tilde{\mathcal{E}})$, where $\Tilde{\mathcal{E}} = \left\lbrace
    \edge{u}{v}{e} : \edge{v}{u}{e} \in \mathcal{E}
    \right\rbrace$ contains the corresponding edges reversing directions.
\end{defin}

We then show that the gradient of g-SpMM and g-SDDMM functions can also be expressed as g-SpMM and g-SDDMM functions.

\begin{lem}
    Assume we are given the g-SDDMM function
    $$
    \mathbf{M} = \gSDDMM_{\mathcal{G}, \phi_m}\left(\mathbf{X}, \mathbf{Y}, \mathbf{W}\right)
    $$
    defined on graph $\mathcal{G}$ with message function $\phi_m$
    and the objective function $\mathcal{L} = \ell(\mathbf{M})$.  There exists a function
    $$
    \phi'_w :
        \mathbb{R}^{|\mathcal{V}| \times d_1},
        \mathbb{R}^{|\mathcal{V}| \times d_2},
        \mathbb{R}^{|\mathcal{E}| \times \left(d_3 + d_4\right)}
     \mapsto \mathbb{R}^{|\mathcal{E}| \times d_3}
    $$
    such that
    $$
    \pp{\mathcal{L}}{\mathbf{W}} = \gSDDMM_{\mathcal{G}, \phi'_w}
    \left(
        \mathbf{X}, \mathbf{Y}, \left[\mathbf{W}; \dpp{\mathcal{L}}{\mathbf{M}}\right]
    \right)
    $$
\end{lem}

\begin{proof}
By chain rule, for each $\edge{u}{v}{e} \in \mathcal{E}$ we have:
\begin{equation*}
    \dpp{\mathcal{L}}{\mathbf{w}_e} = 
    \dpp{\mathcal{L}}{\mathbf{m}_e} 
    \dpp{\mathbf{m}_e}{\mathbf{w}_e} =
    \dpp{\mathcal{L}}{\mathbf{m}_e} 
    \nabla_{\mathbf{w}_e} \phi_m\left( \mathbf{x}_u, \mathbf{y}_v, \mathbf{w}_e \right)
\end{equation*}
\end{proof}

\begin{lem}
    \label{lem:rev}
    
    Assume we are given the g-SDDMM function
    $$
    \mathbf{M} = \gSDDMM_{\mathcal{G}, \phi_m}\left(\mathbf{X}, \mathbf{Y}, \mathbf{W}\right)
    $$
    defined on graph $\mathcal{G}$ with message function $\phi_m$
    and the objective function $\mathcal{L} = \ell(\mathbf{M})$.  There exists functions $\phi'_x$ and $\phi'_y$
    $$
    \begin{gathered}
    \phi'_x:
        \mathbb{R}^{|\mathcal{V}| \times d_1},
        \mathbb{R}^{|\mathcal{V}| \times d_2},
        \mathbb{R}^{|\mathcal{E}| \times \left(d_3 + d_4\right)}
        \mapsto \mathbb{R}^{|\mathcal{V}| \times d_1} \\
    \phi'_y:
        \mathbb{R}^{|\mathcal{V}| \times d_1},
        \mathbb{R}^{|\mathcal{V}| \times d_2},
        \mathbb{R}^{|\mathcal{E}| \times \left(d_3 + d_4\right)}
        \mapsto \mathbb{R}^{|\mathcal{V}| \times d_2}
    \end{gathered}
    $$
    such that
    $$
    \begin{gathered}
    \dpp{\mathcal{L}}{\mathbf{X}} = \gSpMM_{\Tilde{\mathcal{G}}, \phi'_x, \sum}
    \left(
        \mathbf{X}, \mathbf{Y}, \left[\mathbf{W}; \dpp{\mathcal{L}}{\mathbf{M}}\right]
    \right) \\
    \dpp{\mathcal{L}}{\mathbf{Y}} = \gSpMM_{\mathcal{G}, \phi'_y, \sum}
    \left(
        \mathbf{X}, \mathbf{Y}, \left[\mathbf{W}; \dpp{\mathcal{L}}{\mathbf{M}}\right]
    \right)
    \end{gathered}
    $$
    where $\Tilde{\mathcal{G}}$
    represents the reverse graph of $\mathcal{G}$, and
    $\sum$ denotes the summation as a reduce function of the g-SpMMs.

\end{lem}

\begin{proof}
By chain rule, for each $u, v \in \mathcal{V}$ we have:
\begin{align*}
    \dpp{\mathcal{L}}{\mathbf{x}_u} &= \sum_{e', v': \edge{u}{v'}{e'}\in \mathcal{E}}
    \dpp{\mathcal{L}}{\mathbf{m}_{e'}}
    \dpp{\mathbf{m}_{e'}}{\mathbf{x}_u} \\ &= 
    \sum_{e', v': \edge{u}{v'}{e'}\in \mathcal{E}} \dpp{\mathcal{L}}{\mathbf{m}_{e'}}
    \nabla_{\mathbf{x}_u} \phi_m\left( \mathbf{x}_u, \mathbf{y}_{v'}, \mathbf{w}_{e'} \right) \\ &= 
    \underbrace{\sum_{e', v': (v',e',u)\in \Tilde{\mathcal{E}}}}_{\text{reduce function}} \underbrace{\dpp{\mathcal{L}}{\mathbf{m}_{e'}}
    \nabla_{\mathbf{x}_u} \phi_m\left( \mathbf{x}_u, \mathbf{y}_{v'}, \mathbf{w}_{e'} \right)}_{\text{message function on the reverse graph}} \\
    \dpp{\mathcal{L}}{\mathbf{y}_v} &= \sum_{u', e': \edge{u'}{v}{e'}\in \mathcal{E}}
    \dpp{\mathcal{L}}{\mathbf{m}_{e'}}
    \dpp{\mathbf{m}_{e'}}{\mathbf{y}_v} \\ &= 
    \underbrace{\sum_{u', e': \edge{u'}{v}{e'}\in \mathcal{E}}}_{\text{reduce function}} \underbrace{\dpp{\mathcal{L}}{\mathbf{m}_{e'}}
    \nabla_{\mathbf{y}_v} \phi_m\left( \mathbf{x}_{u'}, \mathbf{y}_v, \mathbf{w}_{e'} \right)}_{\text{message function}} \\
\end{align*}
\end{proof}

\begin{lem}
    Assume we are given the g-SpMM function
    $$
    \mathbf{Z} = \gSpMM_{\mathcal{G}, \phi_z, \rho}\left(\mathbf{X}, \mathbf{Y}, \mathbf{W}\right)
    $$
    defined on graph $\mathcal{G}$ with message function $\phi_z$ and reduce function $\rho$
    and the objective function $\mathcal{L} = \ell(\mathbf{Z})$.  There exists a function $\phi'_w$
    $$
    \phi'_w :
        \mathbb{R}^{|\mathcal{V}| \times d_1},
        \mathbb{R}^{|\mathcal{V}| \times \left(d_2 + d_4\right)},
        \mathbb{R}^{|\mathcal{E}| \times d_3}
     \mapsto \mathbb{R}^{|\mathcal{E}| \times d_3}
    $$
    such that
    $$
    \dpp{\mathcal{L}}{\mathbf{W}} = \gSDDMM_{\mathcal{G}, \phi'_w}
    \left(
        \mathbf{X}, \left[\mathbf{Y}; \dpp{\mathcal{L}}{\mathbf{Z}}\right], \mathbf{W}
    \right)
    $$
\end{lem}

\begin{proof}
Let $\mathbf{m}_{e} = \phi_z\left( \mathbf{x}_u, \mathbf{y}_v, \mathbf{w}_{e} \right)$ for all $\edge{u}{v}{e} \in \mathcal{E}$.  Since $\mathbf{m}_{e}$ only participates in computation of $\mathbf{z}_v$, we have
\begin{equation*}
    \dpp{\mathcal{L}}{\mathbf{w}_{e}} =
    \dpp{\mathcal{L}}{\mathbf{z}_{v}}
    \dpp{\mathbf{z}_{v}}{\mathbf{m}_{e}}
    \dpp{\mathbf{m}_{e}}{\mathbf{w}_{e}} =
    \dpp{\mathcal{L}}{\mathbf{z}_{v}}
    \nabla_{\mathbf{m}_{e}}\rho\left(\left\lbrace
    \mathbf{m}_{e'} : (u', e', v) \in \mathcal{E}
    \right\rbrace\right)
    \nabla_{\mathbf{w}_{e}} \phi\left( \mathbf{x}_u, \mathbf{y}_v, \mathbf{w}_{e} \right)
\end{equation*}
\end{proof}

\begin{lem}
    \label{lem:complex_dspmm_dx}
    Assume we are given the g-SpMM function
    $$
    \mathbf{Z} = \gSpMM_{\mathcal{G}, \phi_z, \rho}\left(\mathbf{X}, \mathbf{Y}, \mathbf{W}\right)
    $$
    defined on graph $\mathcal{G}$ with message function $\phi_z$ and reduce function $\rho$
    and the objective function $\mathcal{L} = \ell(\mathbf{Z})$.  There exists functions $\phi'_x$ and $\phi'_y$
    $$
    \begin{gathered}
    \phi'_x:
        \mathbb{R}^{|\mathcal{V}| \times d_1},
        \mathbb{R}^{|\mathcal{V}| \times \left(d_2 + d_4\right)},
        \mathbb{R}^{|\mathcal{E}| \times d_3}
        \mapsto \mathbb{R}^{|\mathcal{V}| \times d_1} \\
    \phi'_y:
        \mathbb{R}^{|\mathcal{V}| \times d_1},
        \mathbb{R}^{|\mathcal{V}| \times \left(d_2 + d_4\right)},
        \mathbb{R}^{|\mathcal{E}| \times d_3}
        \mapsto \mathbb{R}^{|\mathcal{V}| \times d_2}
    \end{gathered}
    $$
    and set functions $\rho'_x$, $\rho'_y$, such that
    $$
    \begin{gathered}
    \dpp{\mathcal{L}}{\mathbf{X}} = \gSpMM_{\Tilde{\mathcal{G}}, \phi'_x, \rho'_x}
    \left(
        \mathbf{X}, \left[\mathbf{Y}; \dpp{\mathcal{L}}{\mathbf{Z}}\right], \mathbf{W}
    \right) \\
    \dpp{\mathcal{L}}{\mathbf{Y}} = \gSpMM_{\mathcal{G}, \phi'_y, \rho'_y}
    \left(
        \mathbf{X}, \left[\mathbf{Y}; \dpp{\mathcal{L}}{\mathbf{Z}}\right], \mathbf{W}
    \right)
    \end{gathered}
    $$
    where $\Tilde{\mathcal{G}}$
    represents the reverse graph of $\mathcal{G}$.
\end{lem}

\begin{proof}
    We first show the correctness for $\pp{\mathcal{L}}{\mathbf{Y}}$.
    
    Let $\mathbf{m}_{e} = \phi_z\left( \mathbf{x}_u, \mathbf{y}_v, \mathbf{w}_{e} \right)$ for all $\edge{u}{v}{e} \in \mathcal{G}$.
    Since $\mathbf{y}_v$ only takes part in the computation of $\mathbf{z}_v$, we have
    \begin{align*}
        \dpp{\mathcal{L}}{\mathbf{y}_v} &=
        \dpp{\mathcal{L}}{\mathbf{z}_v}
        \dpp{\mathbf{z}_v}{\mathbf{y}_v} \\ &=
        \dpp{\mathcal{L}}{\mathbf{z}_v}
        \sum_{u', e': \edge{u'}{v}{e'} \in \mathcal{E}}
          \dpp{\mathbf{z}_v}{\mathbf{m}_{e'}}
          \dpp{\mathbf{m}_{e'}}{\mathbf{y}_v} \\ &=
        \sum_{u',e':\edge{u'}{v}{e'} \in \mathcal{E}}
          \dpp{\mathcal{L}}{\mathbf{z}_v}
          \nabla_{\mathbf{m}_{e'}}\rho\left(\left\lbrace
            \mathbf{m}_{e''} : \edge{u''}{v}{e''} \in \mathcal{E}
            \right\rbrace\right)
          \nabla_{\mathbf{y}_{v}} \phi_z\left( \mathbf{x}_{u'},
            \mathbf{y}_v, \mathbf{w}_{e'} \right)
    \end{align*}
    
    To derive $\pp{\mathcal{L}}{\mathbf{X}}$, we need to sum over all successors of $u$:
    
    \begin{align*}
        \dpp{\mathcal{L}}{\mathbf{x}_u} &=
        \sum_{e',v' : \edge{u}{v'}{e'} \in \mathcal{E}}
        \dpp{\mathcal{L}}{\mathbf{z}_{v'}}
        \dpp{\mathbf{z}_{v'}}{\mathbf{x}_u} \\ &=
        \sum_{e',v' : \edge{u}{v'}{e'} \in \mathcal{E}}
        \dpp{\mathcal{L}}{\mathbf{z}_{v'}}
        \dpp{\mathbf{z}_{v'}}{\mathbf{m}_{e'}}
        \dpp{\mathbf{m}_{e'}}{\mathbf{x}_u} \\ &=
        \sum_{e',v' : \edge{u}{v'}{e'} \in \mathcal{E}}
        \dpp{\mathcal{L}}{\mathbf{z}_{v'}}
        \nabla_{\mathbf{m}_{e'}}\rho\left(\left\lbrace
          \mathbf{m}_{e''} : (u'',e'',v') \in \mathcal{E}
          \right\rbrace\right)
        \nabla_{\mathbf{x}_{u}} \phi_z\left(
          \mathbf{x}_u, \mathbf{y}_{v'}, \mathbf{w}_{e'} \right) \\ &=
        \sum_{e',v':(v',e',u) \in \Tilde{\mathcal{E}}}
          \dpp{\mathcal{L}}{\mathbf{z}_{v'}}
          \nabla_{\mathbf{m}_{e'}}\rho\left(\left\lbrace
            \mathbf{m}_{e''} : (v',e'',u'') \in \Tilde{\mathcal{E}}
            \right\rbrace\right)
          \nabla_{\mathbf{x}_{u}} \phi_z\left( \mathbf{x}_{u},
            \mathbf{y}_{v'}, \mathbf{w}_{e'} \right)
    \end{align*}
    
    We can see that for each node $v$, $\pp{\mathcal{L}}{\mathbf{y}_v}$ is only a function of its own features, the predecessors' features, and the features of inbound edges.  Therefore, $\pp{\mathcal{L}}{\mathbf{Y}}$ can be computed via a g-SpMM defined on graph $\mathcal{G}$.
    
    Similarly, for each node $u$, $\pp{\mathcal{L}}{\mathbf{x}_u}$ is only a function of its own features, the successors' features, and the features of outbound edges.  It can be computed via a g-SpMM defined on the reverse graph $\Tilde{\mathcal{G}}$
\end{proof}

\begin{proof}[Proof of Theorem~\ref{thm:mm}]
    The theorem can be proved trivially from the lemmas above, as the addition of two g-SDDMM functions on the same graph is still a g-SDDMM function on the same graph.  The same holds for g-SpMM functions as well.
\end{proof}

\section{Dataset Statistics}

\begin{table}[h!]
    \hspace{0.9cm}
    \caption{Statistics of all datasets used in Sec.~\ref{sec:eval}}
    \centering
    \begin{tabular}{lrrrr}\toprule
         Dataset & \# Nodes & \# Edges & \begin{tabular}{@{}c@{}}\# Node \\ Features\end{tabular} & \begin{tabular}{@{}c@{}}\# Edge \\ Features\end{tabular}  \\\midrule
         \multicolumn{5}{c}{Node Classification} \\\midrule
         REDDIT & 232,965 & 11,606,919 & 602 & 0 \\
         OGBN-ARXIV & 169,343 &  1,166,243 & 128 & 0 \\
         OGBN-PROTEIN & 132,534 &  39,561,252 & 0 & 8 \\
         OGBN-PRODUCT & 2,449,029 & 61,859,140 & 100 & 0 \\\midrule
         \multicolumn{5}{c}{Link Prediction} \\\midrule
         ML-100K & \begin{tabular}{@{}r@{}}943 users \\ 1,682 movies\end{tabular} & 100,000 &  \begin{tabular}{@{}r@{}}943 (user) \\ 1,682 (movie)\end{tabular} & 0\\
         ML-1M   & \begin{tabular}{@{}r@{}}6,040 users \\ 3,706 movies\end{tabular} & 1,000,209 & \begin{tabular}{@{}r@{}}6,040 (user) \\ 3,706 (movie)\end{tabular} & 0\\
         ML-10M  & \begin{tabular}{@{}r@{}}69,878 users \\ 10,677 movies\end{tabular} & 10,000,054 & \begin{tabular}{@{}r@{}}69,878 (user) \\ 10,677 (movie)\end{tabular} & 0\\
         OGBL-CITATION &  2,927,963 &  30,561,187 & 128 & 0\\
         OGBL-PPA      & 576,289 & 30,326,273 & 58 & 0\\
         \bottomrule
    \end{tabular}
\end{table}

For ML-100K, ML-1M and ML-10M, we use separate one-hot encoding for user and movie nodes.

\section{Experiment configurations}
\subsection{Full graph training}
Here we list the hyper-parameter configurations used in comparing training speed between DGL and PyG (Sec.~\ref{ss:eval-speed}).

\paragraph{Node classification.}
\begin{itemize}[leftmargin=0.5cm]
    \item The GraphSAGE model on \textsc{Reddit} has two layers, each with 16 hidden size and the aggregator is summation.
    \item The GAT model on \textsc{Reddit} has three layers, each with 16 hidden size and one attention head.
    \item The GraphSAGE model on \textsc{ogbn-arxiv} has three layers, each with 256 hidden size and the aggregator is summation.
    \item The GAT model on \textsc{ogbn-arxiv} has three layers, each with 16 hidden size and four attention heads.
    \item The R-GCN model on \textsc{ogbn-protein} has three layers with 32 hidden size. The graph has 8 edge features in the range of $[0, 1]$, which can be viewed as connectivity strength for 8 relations. The RGCN model takes the following formulation:
    $$H^{(l+1)}=\sigma\left(\sum_{r=1}^{8}D_r^{-1}A_rH^{(l)}W_r^{(l)}+H^{(l)}W_0^{(l)}\right),$$
    where $H^{(l)}$ is the node representations after the $l$-th RGCN layer, $\sigma$ is the ReLU activation function, $D_r$ is the degree matrix for relation $r$, $A_r$ is the adjacency matrix for relation $r$, and $W_r^{(l)}, W_0^{l}$ are learnable weights. $H^{(0)}$ is the initial node features. Since the graph does not have raw node features, we use a scalar $1$ for each node. The original $A_r$'s are binary adjacency matrices and they become weighted in the case of \textsc{ogbn-protein}.
\end{itemize}

\paragraph{Link prediction.} The GCMC models on \textsc{ML-100K}, \textsc{ML-1M} and \textsc{ML-10M} all adopt an encoder-decoder architecture as proposed in the original paper. The input is a one-hot encoding of the movie and item nodes. The encoder has one graph convolution layer which projects the input encoding to a layer of 500 units with summation as the message aggregator, and one fully-connected layer which outputs a layer of 75 units. All the models use a bi-linear model as the decoder and the number of basis is set to two.

\subsection{Mini-batch training}
\textbf{Node classification.}
\begin{itemize}[leftmargin=0.5cm]
    \item For training with neighbor sampling on the \textsc{Reddit} graph, we use a batch size of 1024 and the sampling fanouts are 25 and 10 from the first to the last layer. Both the GraphSAGE and the GAT models used have three layers and a hidden size of 16. The GAT model has 8 attention heads.
    \item For training with neighbor sampling on the \textsc{ogbn-product} graph, we use a batch size of 1024 and sampling fanouts of 15, 10, 5 for GraphSAGE, and a batch size of 128 and sampling fanouts of 10, 10, 10 for GAT. Both the GraphSAGE and the GAT models have three layers with a hidden size of 256. The GAT model has 8 attention heads.
    \item For training with cluster sampling on the \textsc{ogbn-product} graph, we first partition the graph into 15000 clusters. The training batch size is 32, meaning each mini-batch contains the induced subgraph of nodes from 32 clusters. Both the GraphSAGE and the GAT models have three layers with a hidden size of 256. The GAT model has 8 attention heads.
\end{itemize}

\textbf{Link prediction.} For all the experiments, we first partition the input graph into 15000 clusters. The training batch size is 256, meaning each mini-batch contains the induced subgraph of nodes from 256 clusters. All the models adopt an encoder-decoder architecture. The encoder models (i.e., GCN or GAT) all have three layers with a hidden size of 256. The GAT models all have one attention head. The decoder model predicts edges by a dot product between the representations of the incident nodes. By default, only one negative sample is generated per positive sample by corrupting one end-point of the edge.

\section{Kernel sensitivity to graph and model configuration}

\begin{figure}
    \centering
    \includegraphics[width=0.6\linewidth]{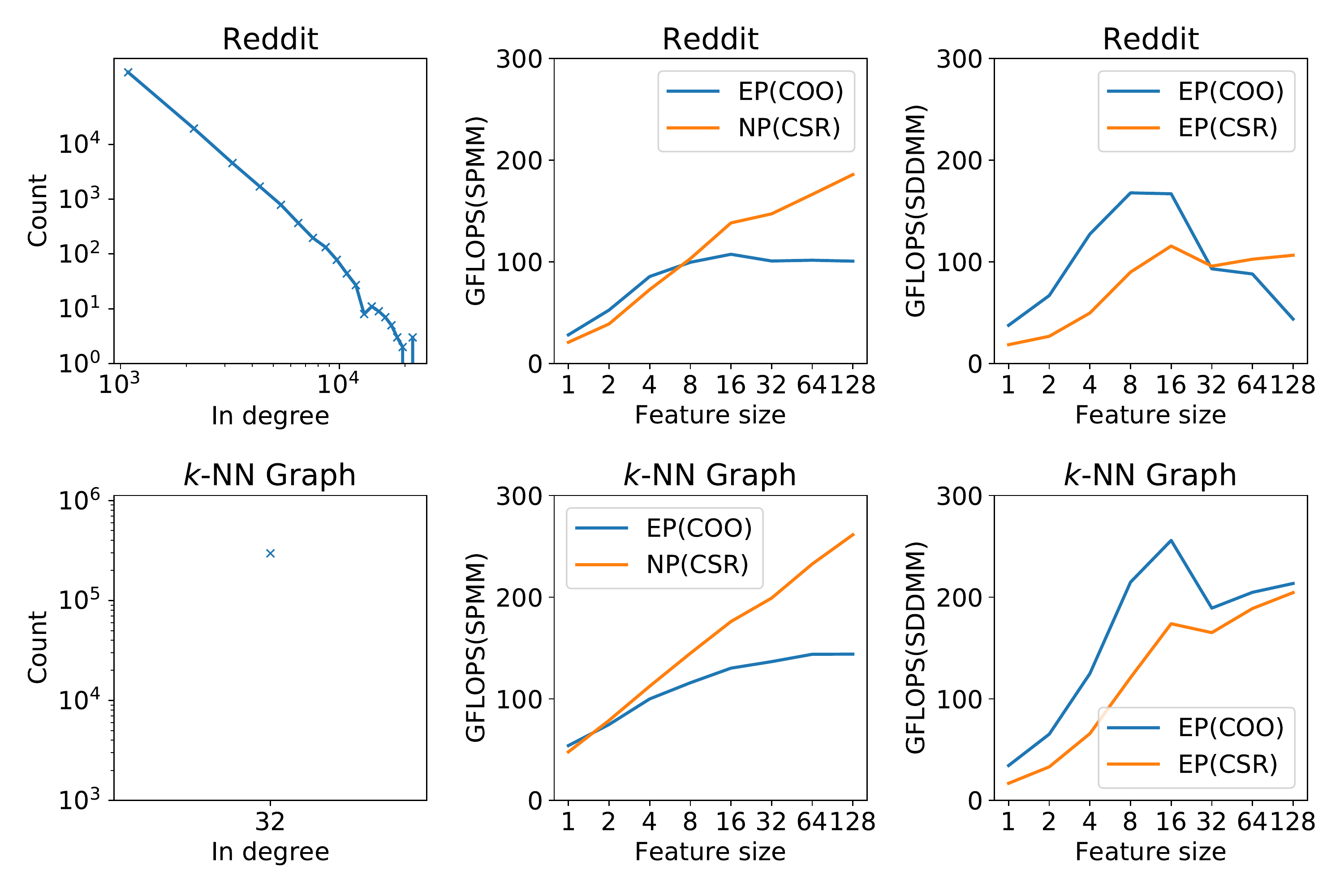}
    \caption{Throughput of SpMM and SDDMM with different parallel strategies, sparse formats on two input graphs. We fix the number of heads to 8 and vary the feature size of each head. The figures in the upper row are for the \textsc{Reddit} dataset while the ones in the lower row are for the $k$-NN graph. NP stands for node parallel while EP stands for edge parallel.}
    \label{fig:spmm-sddmm-format}
\end{figure}

We studied how graph structures and model configurations influence the speed of g-SpMM and g-SDDMM kernels and thus the choice of parallel strategies (i.e., node or edge parallel). We benchmarked a g-SpMM and a g-SDDMM kernel on two input graphs, the \textsc{Reddit} graph from~\citep{graphsage} and a nearest neighbor graph generated by~\citep{qi2017pointnet++}, with varying feature sizes. The two graphs have very different degree frequencies, with \textsc{Reddit} (Figure~\ref{fig:spmm-sddmm-format} upper row) having power-law degree distribution while the $k$-NN graph (Figure~\ref{fig:spmm-sddmm-format} lower row) having a constant indegree equal to 32. The g-SpMM kernel is extracted from the Graph Attention Network (GAT)~\citep{gat} model; it multiplies a neighbor node's representation $\mathbf{x}_u$ with the attention weight $\alpha_e$ on the edge during message aggregation. With multiple attention heads (8 in our experiment), the formulation of each head $h$ is as follows:
\[
\mathbf{z}_{v,h}=\sum_{\edge{u}{v}{e}\in\mathcal{E}}\alpha_{e,h}\mathbf{x}_{u,h}
\]

The g-SDDMM kernel computes the attention weight by a dot-product of the source and destination nodes:

\[
\alpha_{e,h}=\langle \mathbf{x}_{u,h}, \mathbf{x}_{v,h} \rangle, \forall \edge{u}{v}{e}\in\mathcal{E}
\]

The kernel throughput (in GFLOPS) measured on a NVIDIA V100 GPU is shown in Figure~\ref{fig:spmm-sddmm-format}. It shows that the optimal choice of edge parallel or node parallel relies on graph structure, feature size and the sparse format in use. For SpMM, edge parallel is slightly better than node parallel for small feature size but eventually becomes worse when the feature size scales up due to the overhead from atomic aggregation. For SDDMM, edge parallel on COO outperforms CSR by a large margin up to feature size equal to 16 but again becomes worse afterwards because of the better memory locality of CSR.


\section{Guide for porting GNN models across deep learning frameworks}
DGL provides a framework-neutral design and allows users to develop GNN models
on different deep learning frameworks, such as PyTorch, TensorFlow and MXNet.
Assume a user finds an existing GNN model in one framework and wishes to port
it to another one that s/he is familiar with. Such porting needs to
address three categories of differences in the deep learning frameworks.

\begin{figure}[pt]
	\centering
\begin{subfigure}[t]{0.48\textwidth}
\centering
\begin{lstlisting}
from torch import nn

class SAGEConv(nn.Module):
  def __init__(self, in_feat, out_feat,
               feat_drop=0., activation=None):
    pass
  def forward(self, tensor):
    pass
\end{lstlisting}
\caption{PyTorch}
\end{subfigure}
\begin{subfigure}[t]{0.48\textwidth}
\centering
\begin{lstlisting}
from tensorflow.keras import layers

class SAGEConv(layers.Layer):
  def __init__(self, in_feats, out_feats,
               feat_drop=0., activation=None):
    pass
  def call(self, tensor):
    pass
\end{lstlisting}
\caption{TensorFlow}
\end{subfigure}
\caption{Module classes inherit from different classes in PyTorch and TensorFlow.}
\label{fig:inherit}
\end{figure}

\begin{figure}[pt]
	\centering
\begin{subfigure}[t]{0.48\textwidth}
\centering
\begin{lstlisting}
from torch import nn
fc = nn.Linear(in_feat, out_feat)
gain = nn.init.calculate_gain('relu')
nn.init.xavier_uniform_(fc, gain=gain)
\end{lstlisting}
\caption{PyTorch}
\end{subfigure}
\begin{subfigure}[t]{0.48\textwidth}
\centering
\begin{lstlisting}
import tensorflow as tf
from tensorflow.keras import layers
xinit = tf.keras.initializers.VarianceScaling(
            scale=np.sqrt(2), mode="fan_avg",
            distribution="untruncated_normal")
fc = layers.Dense(out_feats, kernel_initializer=xinit)
\end{lstlisting}
\caption{TensorFlow}
\end{subfigure}
\caption{Create a fully connected sub-module in PyTorch and TensorFlow.}
\label{fig:submodule}
\end{figure}

Porting models usually involves in three steps.
\begin{itemize}[leftmargin=0.5cm]
    \item Step 1 is to change
model class inheritance. For example, when porting from Pytorch to TensorFlow,
the model class should inherit from \code{tensorflow.keras.layer.Layer} instead
of \code{torch.nn.Module}. Figure \ref{fig:inherit} shows such an example.
    \item Step 2 is to change the sub-modules used inside the model.
These sub-modules are usually defined in the initialization method of the model class.
Different frameworks usually define similar sub-modules but with different sub-module names
and different arguments. They also initialize the parameters in the sub-modules
differently. For example, the fully connected layer in Pytorch is defined in \code{nn.Linear} but it is defined in \code{layers.Dense} in TensorFlow.
The arguments of the sub-modules are also different. We always need to define
the input and output dimensions for the Pytorch sub-modules but only need to
define the output dimensions for the TensorFlow sub-modules. In addition,
PyTorch initializes parameters after the definition of \code{nn} modules,
while TensorFlow specifies initialization method together with layer definition.
Figure \ref{fig:submodule} shows an example of such differences.
    \item Step 3 is to replace the framework-specific operators. Similar to sub-modules,
different frameworks define similar tensor operators but with different names
and different input arguments. For example, matrix multiplication is defined in
\code{tensorflow.matmul} in TensorFlow and is defined in \code{torch.matmul} in Pytorch.
\end{itemize}

Figure~\ref{fig:graphsage} shows a complete example of porting
GraphSAGE from PyTorch to TensorFlow. We place the code side by side
to contrast the key differences. As shown above,
the PyTorch version of \code{SAGEConv} needs to inherit
from \code{torch.nn.Module} while the TensorFlow version inherits from
\code{tensorflow.keras.layers.Layers}. Most of the modifications
are in the \code{\_\_init\_\_} function of the class, where we change the members
defined as \code{nn} modules to TensorFlow's counterparts. In this example, the \code{forward} function (cf. \code{call} function
in TensorFlow) only invokes DGL's message passing computation via \code{update\_all}.
Because no tensor operators are explicitly invoked, there are no modifications for
tensor operators.

\begin{figure}[pt]
	\centering

\begin{subfigure}[t]{0.48\textwidth}
\centering
\begin{lstlisting}
from torch import nn


class SAGEConv(nn.Module):
  def __init__(self, in_feat, out_feat,
               feat_drop=0., activation=None):
    super(SAGEConv, self).__init__()
    src_feat, dst_feat = expand_as_pair(in_feat)
    self.feat_drop = nn.Dropout(feat_drop)
    self.activation = activation
    self.fc_self = nn.Linear(dst_feat, out_feat)
    self.fc_neigh = nn.Linear(src_feat, out_feat)
    gain = nn.init.calculate_gain('relu')
    nn.init.xavier_uniform_(
        self.fc_self.weight, gain=gain)
    nn.init.xavier_uniform_(
        self.fc_neigh.weight, gain=gain)
        
  def forward(self, graph, feat):
    feat_src = feat_dst = self.feat_drop(feat)
    graph.srcdata['h'] = feat_src
    graph.update_all(fn.copy_u('h', 'm'),
                     fn.mean('m', 'neigh'))
    h_neigh = graph.dstdata['neigh']
    rst = self.fc_self(feat_dst)
             + self.fc_neigh(h_neigh)
    return self.activation(rst)
            
\end{lstlisting}
\caption{PyTorch}
\end{subfigure}
\begin{subfigure}[t]{0.48\textwidth}
\centering
\begin{lstlisting}
import tensorflow as tf
from tensorflow.keras import layers

class SAGEConv(layers.Layer):
  def __init__(self, in_feats, out_feats,
               feat_drop=0., activation=None):
    super(SAGEConv, self).__init__()
    
    self.feat_drop = layers.Dropout(feat_drop)
    self.activation = activation
    xinit = tf.keras.initializers.VarianceScaling(
            scale=np.sqrt(2), mode="fan_avg",
            distribution="untruncated_normal")
    self.fc_self = layers.Dense(
            out_feats, kernel_initializer=xinit)
    self.fc_neigh = layers.Dense(
            out_feats, kernel_initializer=xinit)
        
  def call(self, graph, feat):
    feat_src = feat_dst = self.feat_drop(feat)
    graph.srcdata['h'] = feat_src
    graph.update_all(fn.copy_u('h', 'm'),
                     fn.mean('m', 'neigh'))
    h_neigh = graph.dstdata['neigh']
    rst = self.fc_self(feat_dst)
                 + self.fc_neigh(h_neigh)
    return self.activation(rst)
\end{lstlisting}
\caption{TensorFlow}
\end{subfigure}
\caption{The implementation of GraphSAGE in PyTorch and TensorFlow.}
\label{fig:graphsage}
\end{figure}

\end{appendices}

\end{document}